\pdfoutput=1
\documentclass[format=acmsmall, review=false, screen=true]{acmart}

\usepackage{framed}
\usepackage{balance}
\usepackage{epsfig}
\usepackage{color}
\usepackage{subfigure}
\usepackage{multirow,tabularx}
\usepackage{amssymb}
\usepackage{mathtools}
\usepackage{graphicx}
\usepackage{multirow}
\usepackage{bm}
\usepackage{csquotes}
\usepackage{array}
\usepackage[yyyymmdd,hhmmss]{datetime}
\usepackage[subject={Todo}]{pdfcomment}
\usepackage[textsize=footnotesize,bordercolor=black!20]{todonotes}
\usepackage{xcolor,colortbl}
\usepackage{amssymb}
\usepackage{pifont}
\usepackage[algo2e,titlenumbered,ruled]{algorithm2e} 
\usepackage{lipsum,environ,amsmath}
\usepackage{slashbox,booktabs,amsmath}
\usepackage{hyphenat}

\newcounter{Lcount}
\newcommand{\numsquishlist}{
   \begin{list}{\arabic{Lcount}. }
    { \usecounter{Lcount}
 \setlength{\itemsep}{-.1ex}      \setlength{\parsep}{0ex}
      \setlength{\topsep}{0ex}       \setlength{\partopsep}{0ex}
      \setlength{\leftmargin}{1em} \setlength{\labelwidth}{1em}
      \setlength{\labelsep}{0.1em} } }
\newcommand{\numsquishend}{\end{list}}

\newcommand{\squishlist}{
   \begin{list}{$\bullet$}
    { \setlength{\itemsep}{-.1ex}      \setlength{\parsep}{0ex}
      \setlength{\topsep}{0ex}       \setlength{\partopsep}{0ex}
      \setlength{\leftmargin}{.8em} \setlength{\labelwidth}{1em}
      \setlength{\labelsep}{0.5em} } }
\newcommand{\squishend}{\end{list}}

\newcommand{\cmip}{{\sc Coordination Strategy Inference Problem}\xspace}%

\DeclareMathOperator*{\argmin}{\mathop{\mathrm{argmin}}\limits}

\newenvironment{problem}[1][htb]
  {
   \begin{algorithm2e}[#1]%
  }{\end{algorithm2e}}

\copyrightyear{2019}
\acmYear{2019}
\setcopyright{acmcopyright}
\acmJournal{TKDD}
\acmPrice{15.00}
\acmDOI{10.1145/1122445.1122456}
\acmISBN{978-1-4503-9999-9/18/06}

\begin{document}
\title{Framework for Inferring Following Strategies from Time Series of Movement Data}

  
\author{Chainarong Amornbunchornvej}
\orcid{0000-0003-3131-0370}
\affiliation{\institution{National Electronics and Computer Technology Center}
\city{Pathum Thani}
  \country{Thailand}}
\email{chainarong.amo@nectec.or.th}

\author{Tanya Berger-Wolf}
\affiliation{\institution{University of Illinois at Chicago}
\city{Chicago}\state{IL}
\country{USA}}
\email{tanyabw@uic.edu}

\renewcommand{\shortauthors}{C. Amornbunchornvej et al.}

\begin{abstract}
How do groups of individuals achieve consensus in movement decisions? Do individuals follow their friends, the one predetermined leader, or whomever just happens to be nearby? To address these questions computationally, we formalize \cmip. In this setting, a group of multiple individuals moves in a coordinated manner towards a target path.  Each individual uses a specific strategy to follow others (e.g. nearest neighbors, pre-defined leaders, preferred friends).  Given a set of time series that includes coordinated movement and a set of candidate strategies as inputs, we provide the first methodology (to the best of our knowledge) to infer whether each individual uses local-agreement-system or dictatorship-like strategy to achieve movement coordination at the group level.  We evaluate and demonstrate the performance of the proposed framework by predicting  the direction of movement of an individual in a group in both simulated datasets as well as two real-world datasets: a school of fish and a troop of baboons.  Moreover, since there is no prior methodology for inferring individual-level strategies, we compare our framework with the state-of-the-art approach for the task of classification of group-level-coordination models. The results show that our approach is highly accurate in inferring the correct strategy in simulated datasets even in complicated mixed strategy settings, which no existing method can infer. In the task of classification of group-level-coordination models, our framework performs better than the state-of-the-art approach in all datasets. Animal data experiments show that fish, as expected, follow their neighbors, while baboons have a preference to follow specific individuals.  Our methodology generalizes to arbitrary time series data of real numbers, beyond movement data.
\end{abstract}

%
%



\maketitle

\section{Introduction}

\begin{figure*}[ht!]
\centering
\includegraphics[width=0.8\columnwidth]{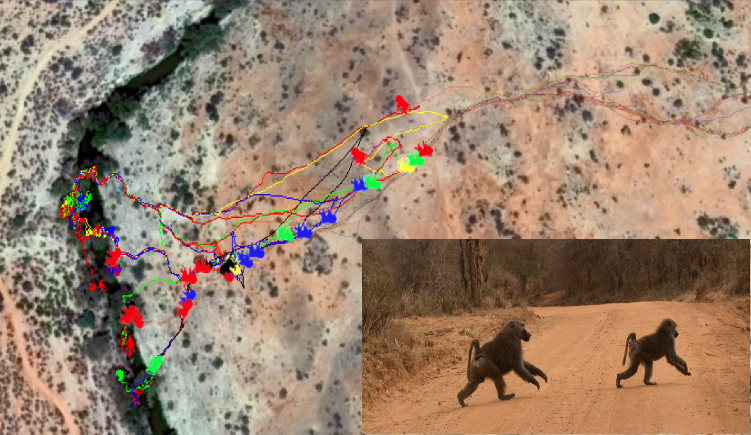}
\caption{An example of GPS-collar trajectories of Olive baboons living in Mpala Research Centre, Kenya~\cite{crofoot2015data,strandburg2015shared}. In this event, the troop is forming coordinated movement.}
\label{fig:Baboon}
\end{figure*}

Coordination is a form of group behavior aimed to make the group achieve a collective goal~\cite{malone1994interdisciplinary}. During the decision-making process, a collective goal is to reach a group's consensus, which is defined as the state when all individuals share a common agreement~\cite{CaoMultiAgent:2013}. One of the mechanisms by which a group can achieve a collective goal is leadership, which is a process of pattern initiation by specific individuals, leaders, then followed by the rest~\cite{FLICAtkdd}. In behavioral studies, coordination problems, such as group decision making, coordinated movement, group hunting, social conflicts, and territorial defense, can be solved by leadership~\cite{Dyer:2009aa,krause2000leadership}. Typically, leaders might not be explicit or global to a group, yet the group can still create coordinated movement via a local strategy (e.g. individuals follow their neighbors)~\cite{Dyer:2009aa}. Moreover, many groups of individuals in Nature have neither leaders nor central authority, but these groups are capable of forming coordination patterns~\cite{valentini2019robots,ray2019information,Hrncir2019}, such as honey bees~\cite{Hrncir2019}, slime molds~\cite{ray2019information}, etc.

\begin{figure}[ht!]
\centering
\includegraphics[width=0.9\columnwidth]{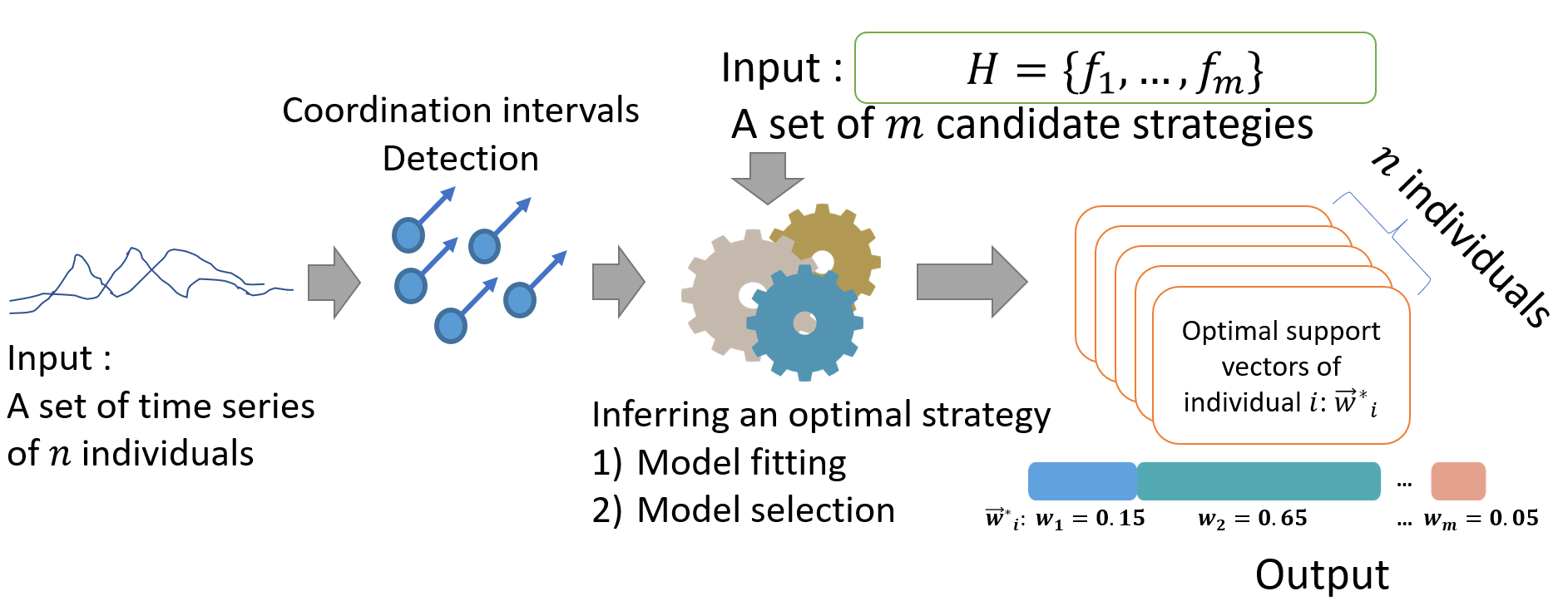}
\caption{An overview of the proposed framework. Given a set of time series as inputs, 1), the framework detects coordination intervals, 2) infers the optimal strategy from a set of candidates that optimally fit the training data, and 3) reports the optimal strategy for each individual from validation data. 
} 

\label{fig:FrameworkOverview}
\end{figure}

In cooperative control of multi-agent systems, the field focuses on how to design a local strategy for each agent so that the group can achieve collective goals~\cite{lewis2013cooperative,CaoMultiAgent:2013,valentini2019robots}. Many systems have been designed by inspiration of natural collective behaviors such as a flock of birds, a school of fish, etc.~\cite{valentini2019robots}. Recently, patterns of opinion formation that emerge from dynamic behaviors of social networks are studied in the view of multi-agent systems~\cite{PROSKURNIKOV2018166,Anderson2019}.

Agents can communicate only with  their neighbors via a communication network, which is defined by any neighborhood concept in some space~\cite{lewis2013cooperative}. There is a large body of work in multi-agent systems that proposes local synchronization strategies~\cite{lewis2013cooperative,CaoMultiAgent:2013,etesami2019simple}. In behavioral studies, the work in ~\cite{Dyer:2009aa,strandburg2013visual} tried to model the coordination process via a concept of information spreading.  A small number of informed agents can spread information through a large number of uninformed agents, which results in the group's consensus and coordinated movement.    The work by Chazelle~\cite{doi:10.1137/100791671} introduced a model, namely a reversible agreement system, that guarantees convergence of the group state, with or without leaders. In more complicated settings, the works in~\cite{amornbunchornvej2018framework,etesami2019simple} provided the analysis of multiagent network systems that can form coordination where networks of relations of agent interactions can change over time. In online social networks, there is also a ``Diffusion Model''~\cite{kempe2003maximizing,goyal2010learning,He2016maximizing} that models an information spreading process among individuals that results in the entire network reaching a common state. 

However, in this paper, we focus on the inverse question of {\em inferring} the local strategies collective individuals use to achieve a state of coordination. There are only a few studies that address this question. The works by Farine {\em et al.}~\cite{farine2016both} found that wild baboons can achieve the state of coordinated movement within a group by following their neighbors or long-term associates, depending on the time scale of the coordination process. 

There are several studies that look at the collective behavior of fish. For example, the work in ~\cite{10.1371/journal.pcbi.1002678} modeled and inferred the rules of movement coordination of fish, which is affected by the group size; Herbert-Read {\em et al.}~\cite{Herbert-Read18726} reported that the rules of movement coordination of fish mainly depend on attraction forces of the group; and Katz {\em et al.}~\cite{Katz18720} showed that fish tend to imitate the direction of neighbors ahead. 
The work in~\cite{10.1371/journal.pcbi.1002961,MEE3:MEE312155} proposed model selection methods to infer the animal-behavior model, but they cannot be used to find models that guarantee coordination.  

\subsection{The current state of the art approach}
The work in ~\cite{FLICAtkdd} provided a framework, FLICA, for leadership inference and model classification in time series data. FLICA considers the shape of time series to infer pairwise relationship who follows whom (instead of considering only directions or positions of individuals). Hence, FLICA subsumes all previous methods~\cite{FLICAtkdd} including FLOCK patterns leadership~\cite{andersson2008reporting}, time-lag following leadership~\cite{kjargaard2013time}, etc. FLICA can infer an underlying possible group model that generated coordination via a classification method. However, FLICA cannot be used to infer individual-level strategies that collectively combine to coordinated movement at the group level. In fact, each individual within a group can use a {\em different} strategy to achieve collective coordination (Proposition~\ref{prop:mixstr}). Hence, in this paper, we develop an approach to fill this methodological gap. Note that we use the words `model', `mechanism', and `strategy' interchangeably. 

\subsection{Our Contributions}
In order to fill the gap in the literature, we formalize \cmip, analyze theoretical properties of a strategy that guarantees coordination, propose hierarchical and non-hierarchical strategies that guarantee coordination, as well as propose a computational framework to infer, from time-series data, individual-level coordination strategies. Given a set of candidate strategies and time series of coordinated movement, our framework is capable of:
\squishlist
\item {\bf Inferring the latent strategies:} inferring the best fit set of mixed or pure strategy for agents that provide the lowest loss value for the task of predicting the direction of movement; and
\item {\bf Movement prediction:} predicting the direction of the next move of each agent when the optimal strategy is unknown, using the set of the inferred latent strategies.
\squishend

We evaluate and demonstrate the  performance of our framework on simulated datasets as well as real-world datasets of animal movement. On simulated data, the task is to infer the correct latent coordination model that was used to generate the simulated time series of coordinated movement.  We use the baboon dataset to predict the next movement to find which strategies each baboon likely used to coordinate its movement. Lastly, in fish datasets, we show how to apply the framework to do the model selection to address a hypothesis about the original model that the fish use to achieve coordinated movement.

\begin{framed}  
\noindent {\cmip:} { To reach a group consensus, individuals have to coordinate with others. There are many strategies each individual can use to achieve coordination at the group level. {\bf Given time series of individual activities and a set of candidate strategies, the goal is to find the set of original strategies individuals used that lead to the group consensus. }}
\end{framed}

\subsection{Flock modeling, dictatorship, and our model-selection framework}
Do agents use some flock models or they use a dictatorship model when they move? In this paper, given a set of candidate models and time series of agents' states (e.g. directions or positions of agents that move in a metric space), the main focus is to develop a model-selection framework for inferring which model(s) are generator of these time series. We focus on two classes of models that the group can reach coordination: neighbors' dependent model and individual-dependent model. For simplicity, the following example is defined the states of agents as directions in movement context.

In neighbors' dependent models, agents move following their group w.r.t. their neighbor directions and positions. This type of model is flock modeling that has a rich literature. The flock models began with the work by CW Reynolds in 1980s~\cite{reynolds1987flocks}. 
The works in~\cite{couzin2002collective,reynolds1987flocks,lopez2012behavioural} proposed flock decentralized models that need no leaders but still be able to  self-organize and maintain coordination.  The works in~\cite{brown2014human,Kerman2012,brown2016two} proposed flock models that humans can control and change a state of group behaviors. The work in~\cite{su2008flocking} proposed a flock multi-agent system with several leaders and showed that the group will converge toward average of leaders' states.
In term of convex hull analysis, for each time step, an agent in neighbors' dependent models changes its state within a convex hull of its neighbors' states except some individuals who lead the flock. There are many state-of-the-art models in flock literature that have rules to make agents avoid collision and other problem. For example , the recent flock model~\cite{qiu2020multi} has been developed for the Unmanned aerial vehicle (UAV) flocking control purpose so that a group can effectively avoid obstacles during a flight. 

Since a state (e.g. direction, velocity, etc.) of movement of each agent in flock models rarely leave a state convex hull of its neighbors, mathematically, according to the works in~\cite{doi:10.1137/100791671,Chazelle2019}, almost all flock models can be viewed as agreement systems studied by Chazelle~\cite{doi:10.1137/100791671} and the recent Averaging system~\cite{Chazelle2019}. Hence, based on Chazelle's works, we propose to use Local Reversible Agreement system (LRA), which is a variation of Chazelle's averaging system, as one of input models of our model-selection framework.

In individual-dependent models, agents move following some specific individuals without any dependency with directions or positions of their neighbors. The obvious case is a dictatorship model where everyone follows leader agents~\cite{FLICAtkdd,amornbunchornvej2018framework,goyal2008discovering}. Influence Maximization models (e.g. linear threshold, independent cascade models)~\cite{He2016maximizing,kempe2003maximizing} are other models that some individuals (influencers) influence other individuals. For these models, the common property is that agents follow some individuals (typically leaders) directly without considering environmental factors (e.g. directions, positions, or velocities of neighbors). In this work, we propose to use a hierarchical model (HM) as a representative model of individual-dependent models to be an input of our model-selection framework. 

Nevertheless, almost all models proposed in the literature assume that all agents are under the same rules when they have to interact with others. In nature, however, different individuals might use different strategies to follow the group but the group still be able to reach coordination. In this work, we propose a framework that can distinguish whether each agent follows its neighbors (LRA), specific individuals (HM), or itself (AR) from time series data. We also show that even though different individuals within the same group use either LRA or HM, the group still be able to reach coordination (see Section~\ref{sec:modelandprop}).

\newcommand{\argmax}{\mathop{\mathrm{argmax}}\limits}

\section{Preliminaries and Definitions}

We use the following notation throughout the paper:

\squishlist
\item $\mathcal{N} =\{1,\dots, n\}$ is a set of agents.
\item $\mathcal{I} \subseteq \mathcal{N}$ is a set of informed agents.
\item $S^t_i$ is a state value of agent $i$ at time $t$, where $S^t_i\in\mathbb{R}^d$. 
\item $S^t = \{S^t_i\}$ is a set of individual states at at time $t$.
\item $S_i = (S^0_i,\dots,S^T_i)$ is a state time series of agent $i$ where $T$ is a length of time series.
\item $S_w = (S^0_w,\dots,S^T_w)$ is a target path where $S^t_w \in \mathbb{R}^d$ is a target state at time $t$.
\item $\mathcal{H}= \{h_i\}$ is a set of strategy functions that agents use to update their current state where $h_i:\mathbb{R}^d  \to \mathbb{R}^d$.
\item $\mathcal{S}=\{S_i\}$ is a set of state time series generated by agents using some set of strategy functions $\mathcal{F}\subseteq\mathcal{H}$.
\item $\sigma \in [0,1]$ is a noise-tolerance threshold.
\squishend

Given a set of $n$ agents $\mathcal{N}$ with a set of their initial states $S^0 = \{S^0_i\}$, these $n$ agents generate a set of state time series $\mathcal{S}=\{S_i\}$, where $S_i= (S^0_i,\dots,S^T_i)$ is the state time series of agent $i\in \mathcal{N}$. For each time step $t$, each agent $i$ updates its state via a strategy function $h_i\in \mathcal{H}$: $S^{t}_i=h_i(S^{t-1}_i)$. However, an informed agent $j\in \mathcal{I}$ always has its state the same as a target path $S_w$: $S^t_j=S^t_w$. 
 
\subsection{Initiator of coordination}
 We use the definitions of coordination, following relation, and coordination initiator from~\cite{FLICAtkdd}.
 Let $\mathcal{S}=\{S_i\}$ be a set of time series. Let $S_{i,t_c}$ denote the time series equal to $S_i$  that starts at time $t_c$, that is $\forall t\in \mathbb{Z}, \: S_{i,t_c}^{t+t_c}=S^t_i$, and $\mathrm{sim}: \mathcal{S}\times \mathcal{S} \to [0,1]$ be any similarity function over time series. We then define the similarity function of a {\em following relation} between two time series (similarity with a time shift):

 \begin{equation}
\mathrm{sim}_{foll}(S_i,S_j) =\max_{\Delta t \in \mathbb{Z}} \mathrm{sim}(S_{i,0},S_{j,0+\Delta t}).
 \label{eq:FollSimFunc}
 \end{equation}

We can also define the minimum time delay of a following relation below:  In Eq.~\ref{eq:FollDelayFunc}, if there are multiple time delays that have the same $\max_{\Delta t \in \mathbb{Z}} \mathrm{sim}(S_{i,0},S_{j,0+\Delta t})$ (similar patterns repeated many times), then we choose the minimum value of these time delays to represent the time delay between two time series that share similar patterns. For example, if a pattern is repeated periodically, Eq~\ref{eq:FollDelayFunc} will ensure that the first iteration will be chosen.

 \begin{equation}
\Delta t_{foll}(S_i,S_j) =\min[\argmax_{\Delta t \in \mathbb{Z}} \mathrm{sim}(S_{i,0},S_{j,0+\Delta t})].
 \label{eq:FollDelayFunc}
 \end{equation}

\begin{definition}[$\sigma$-Following relation]
Let $P=(P^0,\dots)$ and $Q=(Q^0,\dots)$ be time series. If $\mathrm{sim}_{foll}(P,Q) \geq \sigma$ and the time delay $\Delta t_{foll}(P,Q) \geq 0$, then $P$ is followed by $Q$, denoted by $P\preceq Q$. In the case that $\Delta t_{foll}(P,Q) > 0$, then $P$ is strictly followed by $Q$, denoted by $P\prec Q$.
\label{def:follr}
\end{definition}
That is, $Q$ follows $P$ if $Q$ is sufficiently similar to $P$, with a time delay. The $\sigma$ threshold is used to defined the sufficient level of similarity that we accepted that $Q$ follows $P$.

\begin{definition}[Coordination interval]
Let $\mathcal{Q}=\{Q_1,\dots,Q_n\}$ be a set of time series. For any interval $[t_1,t_2]$ if $\forall t \in [t_1,t_2], \:\forall Q_i,Q_j \in \mathcal{Q}$ s.t. $i \neq j$, either $Q_i \preceq Q_j$ or $Q_j \preceq Q_i$, then $[t_1,t_2]$ is a coordination interval.
\end{definition}
That is, a coordination interval is the time when everybody either follows or is followed by somebody.

\begin{definition}[Initiator]
Let $\mathcal{Q}=\{Q_1,\dots,Q_n\}$ be a set of time series and $[t_1,t_2]$ be a coordination interval  of $\mathcal{Q}$. For any $Q_L \in \mathcal{Q}$, if $\forall t \in [t_1,t_2], \:\forall Q_i \in \mathcal{Q} \setminus \{Q_L\}$,  $Q_L \prec Q_i$, then $L$ is an initiator of coordination interval $[t_1,t_2]$.
\end{definition}
The initiator is the one who is followed by everybody during coordination.

\begin{definition}[Coordination event]
Let $\mathcal{Q}=\{Q_1,\dots,Q_n\}$ be a set of time series. If there exists any coordination interval in $\mathcal{Q}$, then $\mathcal{Q}$ is a coordination event.
\end{definition}

\begin{definition}[Coordination strategy]
Let $\mathcal{F} \subseteq \mathcal{H}$ be a set of strategy functions that the agents use to generate a set of state time series $\mathcal{S}=\{S_i\}$.
Each agent $i\in \mathcal{N}$ uses a function $f_i\in \mathcal{F}$ to update its state for each time step. $\mathcal{F}$ is a set of coordination strategies of  $\mathcal{S}$ if $\mathcal{S}$ is a coordination event.
\end{definition}
 
Note that if all agents follow the target path $S_w$, then an informed agent is an initiator of coordination.

\subsection{Problem formalization}

Suppose there is a set of state time series $\mathcal{S}=\{S_i\}$ that was generated by an unknown set of latent coordination strategies  $\mathcal{F} \subseteq \mathcal{H}$ w.r.t. some unknown $\sigma$. The only available inputs are $\mathcal{S}$ and the entire set $\mathcal{H}$. The goal is to find $\mathcal{F}$. The real identity of the target path $S_w$ is unknown, but it is known that $S_w \in \mathcal{S}$. Before formalizing the problem, we define the risk function to measure the fitness of any $h_k \in \mathcal{H}$ that might be in $\mathcal{F}$, for any agent $i$:  

\begin{equation}
risk(S_i,h_k)=\frac{1}{T}\sum^T_{t=1} loss(S^t_i, h_k(S^{t-1}_i)),
\label{eq:RiskFunc}
\end{equation}

where $loss:\mathbb{R}^d\times\mathbb{R}^d \to\mathbb{R} $ is a loss function and $h_k(S^{t-1}_i)$ returns a predicted state $\hat{S}^{t}_i$. Now, we are ready to formalize \cmip. In the next section, we introduce a concept of convergence in multi-agent systems and the relationship between convergence and coordination strategy.

\begin{problem}
    \SetKwInOut{Input}{Input}
    \SetKwInOut{Output}{Output}
    \Input{A set of state time series $\mathcal{S}=\{S_i\}$ generated by multiple agents, where $\mathcal{S}$ is a coordination event; a set of strategy functions $\mathcal{H}=\{h_k\}$; and a loss function $loss:\mathbb{R}^d\times\mathbb{R}^d \to\mathbb{R}$.}
    \Output{A set of minimum risk strategies $\mathcal{F}^*=\{f^*_i\}$ where, for each agent $i$, $f^*_i = \argmin_{h_k \in \mathcal{H}} risk(S_i,h_k)$.}
    \caption{\cmip}
\end{problem}

\section{Models and properties}
\label{sec:modelandprop}
\subsection{Convergence and coordination strategy} 
For the convergence of multi-agent systems, we adopt a notion of $\epsilon$-convergence from \cite{doi:10.1137/100791671}.
\begin{definition}[$\epsilon$-convergence]
Given  $S^0 = \{S^0_1, \dots, S^0_n\}$, a system, which is a set of strategy functions, is said to $\epsilon$-converge if, for $0<\epsilon<1/2$\footnote{

In the work by Chazelle~\cite{doi:10.1137/100791671}, at any time $t$, two agents that move and make a distance between them still less than $1/2$ is considered as a trivial step. The bound $0<\epsilon<1/2$ is defined to ignore microscopic motions.

}, there exists a time constant $t_c>0$ such that for all $t>t_c$, a set of  $n$ agent's states $S^t=\{S^t_1, \dots, S^t_n\}$ can be partitioned into disjoint subsets, where the maximum distance between any pair of agents' states $S^t_i, S^t_j$ from the same subset is less than or equal $\epsilon$. Assuming that a distance function is defined in a metric space.
\end{definition}

\begin{definition}[$\epsilon$-convergence of time series]
Given two time series $S_1, S_2$, we say that  $S_1$  $\epsilon$-converges toward $S_2$ at time $t$ if, for all time $t_c\geq t$, the distance between $S^{t_c}_1$ and $S^{t_c}_2$ is less than or equal $\epsilon$, where $0<\epsilon<1/2$. 
\end{definition}

\begin{proposition}
Suppose $0<\epsilon\leq 1$, if all time series generated by  a set of strategy functions $\mathcal{F} \subseteq \mathcal{H}$ $\epsilon$-converge toward a target path $S_w$, then $\mathcal{F}$ is a set of coordination strategies, where $\sigma=1-\epsilon$. 
\label{prop:siggamewin}
\end{proposition}
\begin{proof}
 Suppose all time series generated by  a set of strategy functions $\mathcal{F} \subseteq \mathcal{H}$ $\epsilon$-converge toward a target path $S_w$. At the converging time $t\in [t_1,\dots]$ every agent's state is within its group convex hull centered at  $S^t_w$ that has the diameter at most $\epsilon$. For some time $t_2\geq t_1$, every time series has a distance between each other at most $\epsilon$. By setting $\sigma=1-\epsilon$, this implies that every time series $\sigma$-follows time series $S_w$. By assigning  all agents that have the state time series the same as $S_w$ to be informed agents, since others follow $S_w$ with some time delay, therefore, we have the $1-\epsilon$-coordination interval $[t_2,\dots]$ and all informed agents are initiators.   
\end{proof}

\begin{proposition}
Let $\mathcal{H}=\{h_k\}$ be a set of pure strategy functions. If all agents use any $h_i \in \mathcal{H}$ as a pure strategy function and their state time series $\epsilon$-converge toward a target path $S_w$, then a  mixed strategy function  $f'$,   created by a linear combination of functions in $\mathcal{H}$, generates a time series that $\epsilon$-converges toward $S_w$. 
\label{prop:mixstr}
\end{proposition}
\begin{proof}
Suppose all functions in $\mathcal{H}$ generate state time series that  $\epsilon$-converge toward $S_w$. At the equilibrium time $t$, when all strategies converge, any strategy in $\mathcal{H}$ that agent $i$ uses ensures that $i$'s state $S^t_i$ is in the convex hull of states centered at $S_w^t$ and has a diameter at most $\epsilon$, since a linear combination of values within a convex hull is still in a convex hull. Therefore, a  mixed strategy function  $f'$  that is created by a linear combination of functions in $\mathcal{H}$ generates a time series that $\epsilon$-converges toward $S_w$.
\end{proof}

\subsection{Convergence models}
\subsubsection{Hierarchical Model Dynamic System (HM)}
\label{sec:HMmodel}

\begin{figure}[ht!]
\centering
\includegraphics[width=.95\columnwidth]{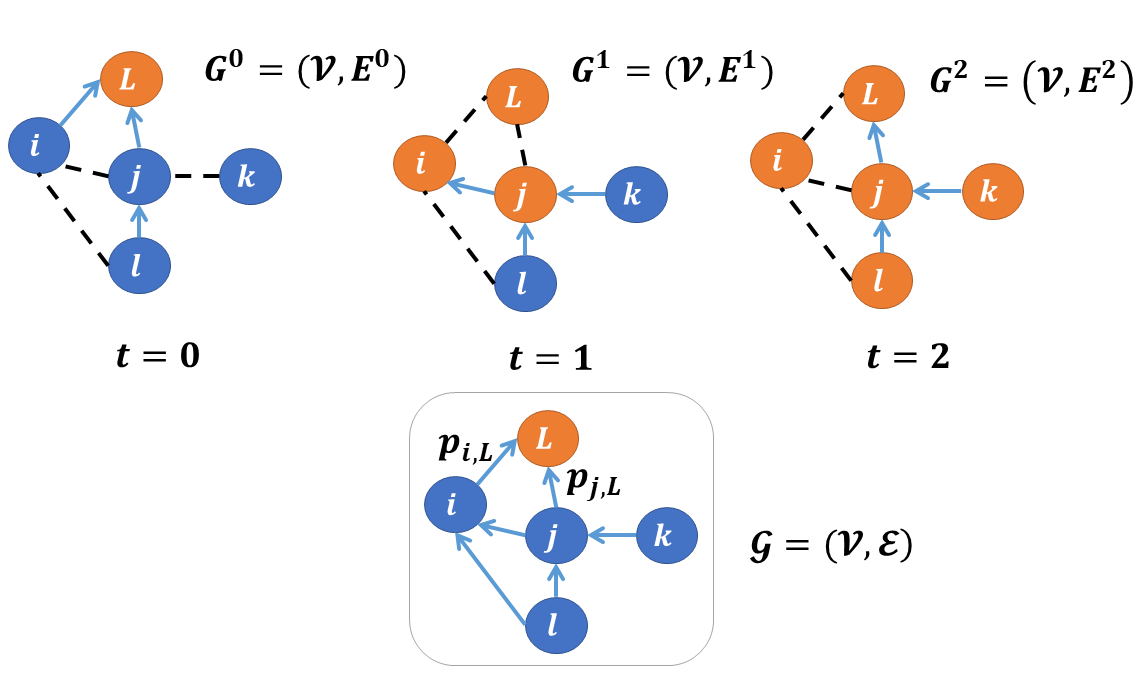}
\caption{An example of communication networks $G^t=(\mathcal{V},E^t)$ between $t=0$ and $t=2$ (above). These networks are the realization of the probabilistic following network $\mathcal{G}=(\mathcal{V},\mathcal{E})$ (below).  The arrows represent the directed edges while the dashed lines are empty edges. When the time step increases, the informed agent $L$ can increasingly spread its state (orange node) to more follower nodes (blue nodes). }
\label{fig:HMex}
\end{figure}

Let $L$ be an informed agent. Let a directed acyclic graph (DAG) $\mathcal{G}=(\mathcal{V},\mathcal{E})$ be graph, where $\mathcal{V}$ is a set of agent nodes and $\mathcal{E}$ is a set of probabilistic edges, so that if $p_{i,j}$ is a probability that $i$ follows $j$ s.t.  $p_{i,j}>0$, then  $(i,j) \in \mathcal{E}$ has the weight $p_{i,j}$. We call $\mathcal{G}=(\mathcal{V},\mathcal{E})$ a probabilistic following network. In this model, $\mathcal{G}$ is connected and every node has a path to a leader node $L$. For every time step $t\geq 0$, the system generates a communication network $G^t=(\mathcal{V},E^t)$, which is a realization of $\mathcal{G}$. The example of the process of generation of a communication network is shown in Fig.~\ref{fig:HMex}. 

Let $S^0=\{S^0_1,\dots,S^0_n\}$ be a set of agent's initial states, $\mathcal{N}^t_i=\{j|(i,j) \in E^t\} \cup \{i\}$ be a set of neighbors of $i$ in $G^t$ that $i$ follows, and $S_w$ be a target path. At any time $t$, the informed agent $L$ updates its state to be $S^t_w$. For any other uninformed agent $i$, it updates the state $S_i^{t}$ according to the aggregation of its neighbors' states. Formally, we have a strategy function for this model as follows: 

\begin{equation}
  f_{\text{HM}}(S^{t-1},i)=\left\{
  \begin{array}{@{}ll@{}}
	S^t_w, & \text{if }i = L \\
   \frac{1}{|\mathcal{N}^t_i|}\sum_{j \in \mathcal{N}^t_i} S^{t-1}_j, & \text{otherwise.}
  \end{array}\right.
\label{eq:HMStrFunc}
\end{equation}

Agents use the above strategy function to update the state $S^t_i=f_{\text{HM}}(S^{t-1},i)$ in this model. In cooperative control literature, the Eq. \ref{eq:HMStrFunc} is called a local voting protocol~\cite{lewis2013cooperative}. A system is known to converge if each communication network $G^t$ stays the same all the time and has a spanning tree that has a leader node $L$ as the root~\cite{lewis2013cooperative}. This is why $\mathcal{G}$ must be connected in order to make a system converge.

\begin{theorem}
Let $S^0=\{S^0_1,\dots,S^0_n\}$ be a set of agents' initial states within Euclidean space. Given a symmetric distance function $\text{DIST}:\mathbb{R}^d\times\mathbb{R}^d\to\mathbb{R}$. If all agents use HM strategy (Eq.~\ref{eq:HMStrFunc}) to update their states, then  all agents' state time series $\epsilon$-converge toward a target state $S^t_w$ with the expectation of the convergence time at most $t_c=n\cdot \max_i ( \text{log}_2(\frac{\text{DIST}(S^0_i,S^0_w)}{\epsilon}) /  p^*)$ time steps if $S^t_w=S^0_w$ for all $t>0$ and $p*= \min_{k,l \in \mathcal{N}, p_{k,l}>0 } p_{k,l}$. 
\label{thrm:HMconv}
\end{theorem}
\begin{proof}

In the first time step, $S^0=\{S^0_1,\dots,S^0_n\}$ forms a convex hull and $S^0_w$ is inside this convex hull because $S^0_w\in S^0$. Given that $L$ is the agent that represents the state of $S^t_w$ where $S^t_L=S^t_w=S^t_0$, for any agent $i$ s.t. $N^t_i=\{L,i\}$, this implies that $i$ is an agent that has no edges to other nodes except $L$ at time $t$.  According to Eq.~\ref{eq:HMStrFunc}, because $\text{DIST}(S^t_L,S^t_i)=1/2 \times (S^t_L+S^t_i)$ and $S^t_L$ is the same all the time, the distance $\text{DIST}(S^t_L,S^t_i)$ reduces by half whenever the link $(i,L) \in E^t$.

Let $T_i \sim \text{Binomial}(t_i,p_{L,i})$ be a random variable of the number of steps it takes until the appearance of a link $(i,L)$ such that $\text{DIST}(S^t_L,S^t_i) = \epsilon$, using $t_i$ trials. We can find the expectation of the time $\mathbb{E}(T_i)$, {\em i.e.,} the expected number of trials $\hat{t}_i$ until  $\text{DIST}(S^t_L,S^t_i) \leq \epsilon$.

From Eq.~\ref{eq:HMStrFunc},
\begin{equation*}
\epsilon = \frac{\text{DIST}(S^0_i,S^0_w)}{2^{T_i}}
\end{equation*}
\begin{equation*}
2^{T_i} = \frac{\text{DIST}(S^0_i,S^0_w)}{\epsilon}
\end{equation*}
\begin{equation*}
{T_i} = \text{log}_2\left( \frac{\text{DIST}(S^0_i,S^0_w)}{\epsilon}\right).
\end{equation*}

Then, by definition of the Binomial expectation,
\begin{equation*}
\mathbb{E}({T_i}) = \hat{t}_i\times  p_{i,L}= \text{log}_2\left({\frac{\text{DIST}(S^0_i,S^0_w)}{\epsilon}}\right).
\end{equation*}

Therefore,
\begin{equation*}
\hat{t}_i = \frac{1}{ p_{i,L}}\text{log}_2\left({\frac{\text{DIST}(S^0_i,S^0_w)}{\epsilon}}\right).
\end{equation*}
In general, we can have an upper bound $t_c \ge \hat{t}_i$ of the expectation of the convergence time as follows:
\begin{equation*}
t_c = n\cdot \max_i \left\{{\frac{1}{ p^*}\text{log}_2\left({\frac{\text{DIST}(S^0_i,S^0_w)}{\epsilon}}\right)}\right\},
\end{equation*}

where
\begin{equation*}
p*= \min_{k,l \in \mathcal{N}, p_{k,l}>0 } p_{k,l}. 
\end{equation*}
\end{proof}
 According to Theorem~\ref{thrm:HMconv} and Proposition~\ref{prop:siggamewin}, if the target path $S_w$ has its target state $S^t_w$ as a fixed point: $S^t_w=S^0_w$ for all $t>0$, then the set of strategy functions $\mathcal{F}$ that contains only HM strategy functions is a set of coordination strategies. In other words, if all agents use $f_{\text{HM}}$ to update their states, then their states converge to a target path. Therefore, a coordination interval exists in their state time series. In contrast, if a target state $S^t_w$ can be changed, the the group still follows the path $S_w$, because only $L$ influences the group and $L$'s state path is $S_w$. However, the convergence might not exist if the difference between two consecutive time steps within the target path is always greater than the group convergent rate.

\subsubsection{Local Reversible Agreement system (LRA)}
\label{sec:LRA}
Let $P^0=\{P^0_1,\dots,P^0_n\}$ be a set of physical points, $S^0$ be a set of initial states, $S_w$ be a target path, $L$ be an informed agent who updates its state in correspondence to $S_w$, and $g(P^{t},S^{t},i)$ be a projection function that agents use to update their physical points. If a state point is a velocity vector, then the projection function is simply the current position plus the velocity vector times the timestep.  First, for $t>0$, we update the physical point  $P_i^{t}=g(P^{t-1},S^{t-1},i)$. Second, we create a set of Delaunay triangulations from $P^t$ to create a communication network $G^t=(\mathcal{V},E^t)$. If $P^t_i$ and $P^t_j$ form the same triangle within the physical space, then $(i,j) \in E^t$. Third, we update a state of each agent based on the structure of  $G^t$.  The example of how to find the neighbors of each individual in LRA is in Fig.~\ref{fig:LRAex}, which defines physical points as positions of individuals and states as movement directions.

Given a triangulation membership function $\delta$ and a set of all points $P^t=\{P^t_i\}$. The function $\delta(P^t_i,P^t_j)=1$ if  $P^t_i,P^t_j$ are members of a triangulation s.t. no other points in $P^t$ are in the triangulation (note that $\delta(P^t_i,P^t_i)=1$), otherwise it is zero.  We have a strategy function for LRA as follows.

\begin{figure}[ht!]
\centering
\includegraphics[width=.6\columnwidth]{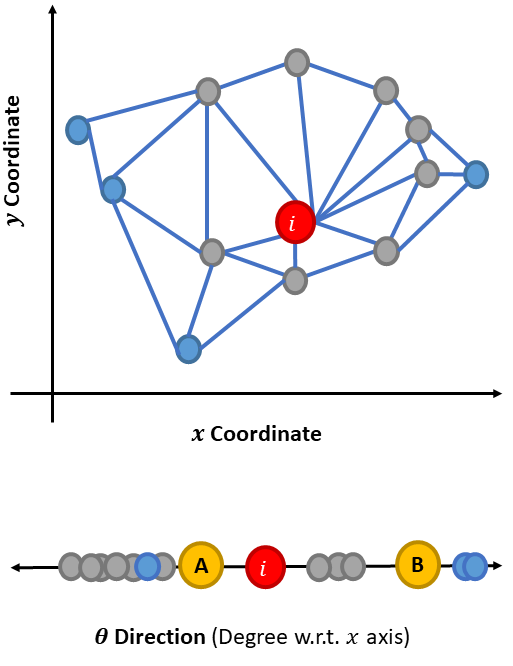}
\caption{ An example of physical points as positions and state points as directions.  In position space (above), the individual $i$ (red node) has all gray nodes as its neighbors in LRA since they are neighbors in  Delaunay triangulation. In the direction space (below), $i$ updates its next direction to be A rather than B since B is outside the $i$'s neighbor convex hull, and A is the averages of grey nodes. }
\label{fig:LRAex}
\end{figure}

\begin{equation}
  f_{\text{LRA}}(P^t,S^{t-1},i)=\left\{
  \begin{array}{@{}ll@{}}
	S^t_w, & \text{if }i = L \\
   \frac{\sum_{j} S^{t-1}_j\cdot \delta(P^t_i,P^t_j)}{\sum_j \delta(P^t_i,P^t_j) }, & \text{otherwise}
  \end{array}\right.
\label{eq:LRAStrFunc}
\end{equation}

The difference between $f_{\text{HM}}$ (Eq.~\ref{eq:HMStrFunc}) and $f_{\text{LRA}}$ (Eq.~\ref{eq:LRAStrFunc}) is that $f_{\text{HM}}$ infers the next state based on a {\em fixed} structure of a probabilistic following network $\mathcal{G}$, independently from the physical space $P^{t}$, whereas $f_{\text{LRA}}$ predicts the next state based on the physical space $P^{t}$.  In other words, $f_{\text{HM}}$ represents an assumption that an agent follows a fixed set of specific individuals w.r.t. the preference graph $\mathcal{G}$ regardless of their relative physical position, while $f_{\text{LRA}}$ represents an assumption that an agent follows anyone who happens to be around without any preference to follow specific individuals. The next theorem  shows that the Local Reversible Agreement is $\epsilon$-convergent.

\begin{theorem}[Chazelle 2011\cite{doi:10.1137/100791671}]
For any $0<\epsilon\leq \rho/n$, an $n$-agent reversible agreement system is $\epsilon$-converged in time $\mathcal{O}(\frac{1}{\rho}\cdot n^2 \text{log}_2(\frac{1}{\epsilon}) $). Where $\rho>0$ is the time-independent agreement parameter corresponding to the system.\footnote{In a Bidirectional agreement system, which is a general model of a reversible agreement system, the $\rho>0$ condition is a necessary condition to make systems converging~\cite{doi:10.1137/100791671}.}
\label{thrm:LRAconv}
\end{theorem}

According to the work by Chazelle~\cite{doi:10.1137/100791671}, LRA is still converged even if one of the agents does not update. In our case, if $S^t_w$ is the same for every time step, then the fixed agent is $L$ who always has $S^t_L = S^0_w$. 
\begin{corollary}
The $n$-agent LRA that has $G^t$ being created from Delaunay triangulation sets converges to a single point. 
\label{corr:LRAconv}
\end{corollary}
\begin{proof}
The graph $G^t$ that is built from Delaunay triangulation is always connected. For each time step, each agent converges to the center of the neighbors' convex hull. Since everyone is connected and the system is $\epsilon$-converge, by transitivity, the entire group converges to the single point.  
\end{proof}

In fact, if the fixed point is $S^0_w$, then, at the equilibrium point, all states form a convex hull around $S^0_w$ with the diameter at most $\epsilon$~\cite{doi:10.1137/100791671}. In contrast, if  $S^t_w$ is not always the same, then the group moves following $S^t_w$ with some time delay.

The Corollary~\ref{corr:LRAconv} tells us that if we follow our physical neighbors (e.g. directions) and everyone does the same thing, the entire group will reach the same consensus (moving to the same direction). In general, if $G^t$ is strongly connected, everyone follows neighbors in $G^t$, and there is one individual $L$ who never follows anyone, then the group converges to $L$'s state. Additionally,  Corollary~\ref{corr:LRAconv} is always true in any metric space where a Delaunay triangulation exists.

According to Corollary~\ref{corr:LRAconv} and Proposition~\ref{prop:siggamewin}, if a target state never changes: $S^t_w=S^0_w$ for all $t>0$, then the set of strategy functions $\mathcal{F}$ that contains only LRA strategy functions is a set of coordination strategies.

\subsubsection{Discussion}
According to Theorem~\ref{thrm:HMconv}, Corollary~\ref{corr:LRAconv}, Proposition~\ref{prop:siggamewin}, and Proposition~\ref{prop:mixstr}, if the data has coordination behaviors, then either HM, LRA, or a mix of those strategies may be the cause of the coordination. However, the question still remains regarding how to infer which strategy is the cause of the coordination. In the next section, we propose a solution to address this question.

\subsection{Non-coordination strategy: Autoregressive-moving-average model} 
Beside agents change their states randomly, Autoregressive-moving-average model (ARMA)~\cite{whitle1951hypothesis} is a strategy that agents change states based on their own states in the past. ARMA is a strategy that has no guarantee that if all agents use this strategy, then the entire group will converge to any state. Formally, given $S^t_i$ is a time series of agent $i$ at time $t$, ARMA model of $S^t_i$ can be represented by a function below:

\begin{equation}
\label{eq:ARMA}
    S^t_i= c+\gamma^t+\sum_{a=1}^p \varphi^a S^{t-a}_i+ \sum_{b=1}^q \theta^b \gamma^{t-b},
\end{equation}

where $\gamma^t$ is a white-noise term at time $t$, $\varphi^1,\dots,\varphi^p$ and $\theta^1,\dots,\theta^q$ are parameters of the model, and $c$ is a constant. The ARMA model represents that an agent state $S^t_i$ has dependency from its own states in the past with some noise. In the Autoregressive model (AR), the term $\sum_{b=1}^q \theta^b \gamma^{t-b}$ in Eq.~\ref{eq:ARMA} is omitted. For simplicity, in our paper, we study AR model only the case that $S^t_i$ is the average of its $p$ states in the past.
\section{Method}
\label{sec:Method}
We are now ready to formally state our approach of inferring movement coordination strategies of agents represented by a collection of time series.


\subsection{Setting}
\begin{figure}[ht!]
\centering
\includegraphics[width=.6\columnwidth]{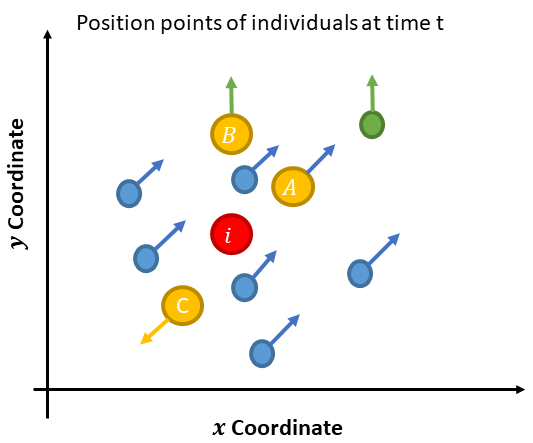}
\caption{ An example of movement strategy inference for $i$. Given the information on positions and directions of individuals in the past (blue and green nodes), we want to infer the $i$'s strategy of movement that can be whether that $i$'s next direction  follows its neighbors (A node), or follows specific individuals (B node), or neither (C node). }
\label{fig:Direx}
\end{figure}

We define a movement direction as a state, but the approach generalizes to arbitrary definitions of states that are defined on Euclidean space.  Hence, $\mathcal{S}_k$ is a set of time series of direction.
We use direction, rather than position, to define the state of an individual and the proxy for collective coordination. The main reason is that directional coordination is common in biology. For example, in~\cite{Katz18720}, the authors report that a fish tends to imitate the direction of neighbors ahead to form collective movement, and other examples abound.
Secondly, synchronization to the same direction implies a collective movement while synchronization to the same position implies staying in the same position without movement. In this paper, we focus on coordination of movement, therefore, we cannot use positions as states to infer strategies of movement.
The final reason for defining states as directions is to use a dimension independent of the positions, which we use to define states of individual strategies. We need to differentiate between the strategy that an individual follows specific individuals' direction regardless of their physical neighbors' choices of direction versus the strategy that an individual follows their physical neighbors' direction without any preference to follow specific individuals.
We assume the following are given as inputs: a set of possible strategy functions $\mathcal{H}$, a collection of position-time-series sets $\mathbb{P}=\{\mathcal{P}_k\}$, and a collection of direction-time-series sets $\mathbb{S}=\{\mathcal{S}_k\}$, where $\mathcal{P}_k=\{P_1,\dots,P_n\}$ and  $\mathcal{S}_k=\{S_1,\dots,S_n\}$.  The data record of $k$th coordination event consists of a pair of $\mathcal{P}_k,\mathcal{S}_k$ that were generated by $n$ agents moving in two-dimensional position space to form directional coordination; all agents coordinately move to the same direction in this interval. Each $\mathcal{S}_k$ contains a coordination interval. The goal to to infer the set of strategy functions $\mathcal{F}\subseteq\mathcal{H}$ that generated $\mathcal{S}_k$. The framework overview is in Fig.~\ref{fig:FrameworkOverview}.

For simplicity of the exposition, we deploy three strategy functions for our framework: HM, LRA, and Auto regressive model (AR). Again, other candidate strategies are admissible. However, these three strategies are canonical exemplars since they make it possible to determine whether the strategy functions that generated a time series of directions of each agent is more hierarchical (HM), or it is more dependent on the physically proximity neighbors (LRA), or it is just a simple function of the agent's past history, independent of its neighbors. We separate $\mathbb{P}$ and $\mathbb{S}$ to be a training part, $(\mathbb{P}_\text{train} \subset \mathbb{P},\: \mathbb{S}_\text{train} \subset \mathbb{S})$, to perform a model fitting, and a validation part, $(\mathbb{P}_\text{val}= \mathbb{P} - \mathbb{P}_\text{train},\: \mathbb{S}_\text{val}= \mathbb{S} - \mathbb{S}_\text{train})$, to perform a model selection. In the case that the input is only a single physical time series $\mathcal{P}$, we use FLICA framework~\cite{FLICAtkdd} to find coordination events and treat each event as a single $\mathcal{P}_k$. Hence, we have $\mathbb{P}$ containing multiple coordination events from $\mathcal{P}$. Then, we create a set of direction-time-series sets $\mathbb{S}$ from $\mathbb{P}$.
The example of movement strategy inference is in Fig.~\ref{fig:Direx}

\subsection{Model fitting}
\setlength{\intextsep}{0pt}
\IncMargin{1em}
\begin{algorithm2e}
\caption{ModelFittingFunction}
\label{algo:ModelFittingFunction}
\SetKwInOut{Input}{input}\SetKwInOut{Output}{output}
\Input{ Position-time-series sets  $\mathcal{P}_\text{train}$,  Direction-time-series  $\mathcal{S}_\text{train}$, and a Threshold $\vec{\kappa}$. }
\Output{Support vectors: $\vec{w}^*_1,\dots,\vec{w}^*_n$. }
\begin{small}
\SetAlgoLined
\nl Let $T$ be a time length of $\mathcal{S}_\text{train}$\;
\nl Inferring dynamic following network $G_{foll}$  from $\mathcal{S}$ using FLICA~\cite{FLICAtkdd}\;
\nl Inferring a global-leadership ranking $R_{L}$ from  $G_{foll}$\;
\nl Aggregating $G_{foll}$ and pruning its edges using $R_{L}$ to create a DAG probabilistic network $\mathcal{G}$\;
\nl \For{ $i=1$ to $n$}
{
    \For{ $t=1$ to $T$}
    {
\nl    Inferring $f'_{\text{HM}}(S^{t-1},i)= \frac{ S^{t-1}_i + \sum_{(i,k) \in \mathcal{E}} p_{i,k}\cdot S^{t-1}_k }{1+\sum_{(i,k) \in \mathcal{E}} p_{i,k}}$ (Eq~\ref{eq:HMstrFunc2} w.r.t. $\mathcal{G}$) to predict $S^{t}_i$\;
\nl    Inferring $f_{\text{LRA}}(P^t,S^{t-1},i)=\frac{\sum_{j} S^{t-1}_j\cdot \delta(P^t_i,P^t_j)}{\sum_j \delta(P^t_i,P^t_j) }$ (Eq~\ref{eq:LRAStrFunc})  to predict $S^{t}_i$ \;
\nl    Inferring $f_{\text{AR}}(S^{t-1},i)=(1/p)\times\sum_{a=1}^p S^{t-a}_i$ to predict $S^{t}_i$ for auto regressive model\;
    }
}
\nl Running Interior point algorithm~\cite{doi:10.1137/S1052623497325107} to solve Problem~\ref{eq:opteq} in order to find $\vec{w}^*_1,\dots,\vec{w}^*_n$\;
\nl Return $\vec{w}^*_1,\dots,\vec{w}^*_n$\;
\end{small}
\end{algorithm2e}\DecMargin{1em}

 We concatenate all time series in $\mathbb{P}_\text{train}$ to be a single time series $\mathcal{P}_\text{train}$ and also concatenate $\mathbb{S}_\text{train}$ to be $\mathcal{S}_\text{train}$. Then we use $\mathcal{P}_\text{train},\: \mathcal{S}_\text{train}$ to perform model fitting. 

Before proceeding with the model fitting, the HM strategy function requires a probabilistic following network $\mathcal{G}=(\mathcal{V},\: \mathcal{E})$. We infer $\mathcal{G}$ from $\mathcal{S}$ by using FLICA~\cite{FLICAtkdd} to create a dynamic following network of $\mathcal{S}_\text{train}$. In this paper, the time window threshold of FLICA has been set at $\omega = 60$ time steps. In the next step, we find a global-leadership ranking, then we aggregate and normalize this dynamic network to be a DAG probabilistic network, such that the high-rank agents do not have a probabilistic following edge to low-rank agents in $\mathcal{G}$. After we have $\mathcal{G}$, we calculate $f'_{\text{HM}}$ as follows:

\begin{equation}
\label{eq:HMstrFunc2}
f'_{\text{HM}}(S^{t-1},i)= \frac{ S^{t-1}_i + \sum_{(i,k) \in \mathcal{E}} p_{i,k}\cdot S^{t-1}_k }{1+\sum_{(i,k) \in \mathcal{E}} p_{i,k}},
\end{equation}

where $ p_{i,k} \in [0,1]$ is a probabilistic weight of edge $(i,k) \in \mathcal{E}$. For the LRA strategy function, we use  the same function as in Eq~\ref{eq:LRAStrFunc}. Lastly, we apply auto regressive model to fit on $\mathcal{S}$ to represent $f_{\text{AR}}$. The $f_{\text{AR}}$ predicts the next state of the agent w.r.t. the average of the states from previous $p$ steps in $\mathcal{S}$. In this paper, we set $p=5$. As mentioned before, we focus on three strategy functions: $f_{\text{HM}}$ (Eq~\ref{eq:HMstrFunc2}), $ f_{\text{LRA}}$ (Eq~\ref{eq:LRAStrFunc}), and $f_{\text{AR}}$. We can view them as a mixed strategy, given a support vector $\vec{w}=[w_1,w_2,w_3]^\text{T}$. 

\begin{equation}
f_{\text{mix}}(a,\vec{w}) = w_1f'_{\text{HM}}(a)+ w_2f_{\text{LRA}}(a)+w_3f_{\text{AR}}(a)
\label{eq:Mix3Str}
\end{equation}
Here $a=(P^t,S^{t-1},i)$, $w_1$ is a support of HM, $w_2$ is a support of LRA, $w_3$ is a support of an auto regressive model, and $w_1,w_2,w_3 \in [0,1]$. We use the sum square error (SSE) as our loss function. Our main goal is to find $\vec{w}^*$ that minimizes $risk(\mathcal{P},\mathcal{S},\vec{w}^*,i)$ below:

\begin{equation}
risk(\mathcal{P},\mathcal{S},\vec{w},i) = \sum_{t=1}^T (D(t))^\text{T}\cdot(D(t)),
\end{equation}
where $D(t)=|f_{\text{mix}}(P^t,S^{t-1},i,\vec{w}) - S^t_i|$ is a difference between predicted and actual direction in which agent $i$ moved at time $t$. For each agent $i$, given $\mathcal{P}_\text{train},\mathcal{S}_\text{train}$ and a threshold vector $\vec{\kappa}=[k_1,k_2,k_3]^\text{T}$, we can find the optimal support vector $\vec{w}^*$ as the optimization problem:

\begin{equation}
\begin{aligned}
& \underset{\vec{w}}{\text{minimize}}
& & risk(\mathcal{P}_\text{train},\mathcal{S}_\text{train},\vec{w},i)  \\
& \text{subject to}
& &  w_i \geq \kappa_i, \; w_i \in \vec{w}, \kappa_i \in \vec{\kappa}.\\
& & & \sum_i w_i =1 \\
& & & w_i,\kappa_i \in [0,1].
\end{aligned}
\label{eq:opteq}
\end{equation}

We use the Interior point algorithm~\cite{doi:10.1137/S1052623497325107}, which is a large-scale algorithm, to solve Problem~\ref{eq:opteq}, which can be consider as a constrained linear least-squares problem. A threshold $\vec{\kappa}$ represents a model bias toward specific strategies. For example, if we have prior information that, with high probability, an agent $i$ uses LRA strategy function, then we can set $\kappa_2 =0.5$ to enforce the optimizer to vary the support $w_2$ within $[0.5,1]$ interval instead of the $[0,1]$ interval. The benefit of having $\vec{\kappa}$ is to prevent  overfitting.  For any agent $i$, suppose $\vec{w}^*_{i,k}$ is the optimal solution of an optimization problem~\ref{eq:opteq} w.r.t. $\vec{\kappa}_k$, then we call $(\vec{w}^*_{i,k}, \vec{\kappa}_k)$ a model. The pseudo code of the model fitting is given in Algorithm~\ref{algo:ModelFittingFunction}.

\subsection{Model selection} 

\setlength{\intextsep}{0pt}
\IncMargin{1em}
\begin{algorithm2e}
\caption{ModelSelectionFunction}
\label{algo:ModelSelectionFunction}
\SetKwInOut{Input}{input}\SetKwInOut{Output}{output}
\Input{ Position-time-series sets  $\mathcal{P}_\text{train}$,$\mathcal{P}_\text{val}$, and Direction-time-series  $\mathcal{S}_\text{train}$,$\mathcal{S}_\text{val}$. }
\Output{Support vectors: $\vec{w}^*_1,\dots,\vec{w}^*_n$. }
\begin{small}
\SetAlgoLined
\nl Setting $\mathcal{K}=\{\vec{\kappa}_k\}$\;
\nl \For{ Each $\vec{\kappa}_k$ in $\mathcal{K}$}
{
Inferring models, $(\vec{w}^*_{1,k}, \vec{\kappa}_k),\dots,(\vec{w}^*_{n,k}, \vec{\kappa}_k)$, from  $\mathcal{P}_\text{train}$,$\mathcal{S}_\text{train}$ using Algorithm~\ref{algo:ModelFittingFunction} \;

}
\nl Finding the optimal support vectors, $\vec{w}^*_1,\dots,\vec{w}^*_n$, from  $\mathcal{P}_\text{val}$,$\mathcal{S}_\text{val}$ using Eq.~\ref{eq:OPTsup}\;
\nl Return $\vec{w}^*_1,\dots,\vec{w}^*_n$\;
\end{small}
\end{algorithm2e}\DecMargin{1em}

First, we vary $\vec{\kappa}_k$ and  find a model $(\vec{w}^*_{i,k}, \vec{\kappa}_k)$ for each agent $i$ from $\mathcal{P}_\text{train},\: \mathcal{S}_\text{train}$. As the result, we have  a set of models $\Phi_i=\{(\vec{w}^*_{i,k}, \vec{\kappa}_{i,k})\}$ that is now used to perform model selection for an agent $i$. We concatenate all time series in $\mathbb{P}_\text{val}$ to be a single time series $\mathcal{P}_\text{val}$ and also concatenate $\mathbb{S}_\text{val}$ to be $\mathcal{S}_\text{val}$.  Finally, for each agent $i$, we find the optimal support vector $\vec{w}^*_i$ using the equation below:

\begin{equation}
   \vec{w}^*_i = \argmin_{ (\vec{w}^*_{i,k}, \vec{\kappa}_k) \in \Phi_i} risk(\mathcal{P}_\text{val},\mathcal{S}_\text{val},\vec{w}^*_k,i). 
   \label{eq:OPTsup}
\end{equation}

After we get the support vector $\vec{w}^*_i=[w^*_{i,1},w^*_{i,2},w^*_{i,3}]^\text{T}$, if $w^*_{i,1}$ is the highest support in $\vec{w}^*_i$, then we say that agent $i$ uses the HM strategy function to coordinate with its group. If $w^*_{i,2}$ has the highest support, then we say that $i$ follows its physical neighbors to coordinate with the group. If  $w^*_{i,3}$ has the highest support, then $i$ just follows its own linear path independently, and if $i$'s path is the target path $S_w$ then $i$ is an informed agent.  Lastly, if at least two of $w^*_{i,1}$, $w^*_{i,2}$, $w^*_{i,3}$ show significantly high weights, then we conclude that $i$ uses a mixed strategy. The pseudo code of the model selection is given in Algorithm~\ref{algo:ModelSelectionFunction}.
\section{Experimental setup}
We test our approach both on simulated and on biological data.

\subsection{Simulations}
\label{sec:simdata}
We generated a set of time series of 2-dimensional positions $\mathcal{P}$  by four different sets of strategy functions: $\mathcal{H}_\text{HM}$,$\mathcal{H}_\text{LRA}$,$\mathcal{H}_\text{HM\&LRA}$, and $\mathcal{H}_\text{MIX}$. A dataset of $\mathcal{H}_\text{HM\&LRA}$ is the dataset that contains some agents that use $\mathcal{H}_\text{HM}$ while some other agents use $\mathcal{H}_\text{LRA}$. In $\mathcal{H}_\text{HM\&LRA}$, when agents use either $\mathcal{H}_\text{HM}$ or $\mathcal{H}_\text{LRA}$, they use the same strategy for all time steps. In contrast, a dataset of $\mathcal{H}_\text{MIX}$ has agents that alternately choose to use between $\mathcal{H}_\text{HM}$ and $\mathcal{H}_\text{LRA}$ w.r.t. some probability. Hence, agents within $\mathcal{H}_\text{MIX}$ do not use the same strategy for all time steps.  We define a set of state-time-series $\mathcal{S}=\{S_i\}$ as a set of time series of directional degrees of $\mathcal{P}=\{P_i\}$, where $P_i=(P^0_i,\dots,P^T_i)$ is a time series of positions of an agent $i$; $S_i=(S^0_i,\dots,S^T_i)$ is time series of directional degrees of an agent $i$ derived from a position time series $P_i$; and $S^t_i \in (-180,180]$ is a degree angle between a direction vector $\vec{v}^t_i=P^t_i-P^{t-1}_i$ and $x$-axis direction vector $[1,0]^\text{T}$. Note that we need to be careful also of the distance between any $S^t_i$ and $S^t_j$ since $-179^\circ$ and $180^\circ$ have a difference of 359 degrees but very similar implications for coordination.

\begin{equation}
  \text{DIST}_\text{dir}(S^t_i,S^t_j)=\left\{
  \begin{array}{@{}ll@{}}
	|S^t_i-S^t_j|, & \text{if }|S^t_i-S^t_j|\leq 180 \\
   360-|S^t_i-S^t_j|, & \text{otherwise}
  \end{array}\right.
\label{eq:DegreeDifFunc}
\end{equation}
 Where $\text{DIST}_\text{dir}(S^t_i,S^t_j) \in [0,180]$. We have only $\mathcal{P}$ as an input for our framework since we can create $\mathcal{S}$ from $\mathcal{P}$. In all simulated datasets, there are 20 agents and ID(1) is the informed agent. ID(1) creates the target path by uniformly and randomly choosing a fixed direction $S^0_w$ as the initial state, then continuing to move in the direction of $S^0_w$ until the end of coordination.
 
\subsubsection{Hierarchical Model Dynamic System}
 In this system, we used a set of strategy function $\mathcal{H}_\text{HM}=\{f_i\}$ to generate $\mathcal{P}_\text{HM}$ where all $f_i$ is $f_\text{HM}$ (Eq.~\ref{eq:HMStrFunc}).  The parameter in this model is the following probability $\rho \in [0,1]$. We set the probability weight of all edges in a probabilistic following network $\mathcal{G}$ equal to $\rho$. The communication network $G^t$ generated by $\mathcal{G}$ is  used to update the directional state $S^t_i$ by the strategy function $f'_\text{HM}$. All 19 agents always follow only ID(1) with the probability $\rho$. In other words, all nodes have edges to ID(1) with the weight $\rho$ in $\mathcal{G}$.  For each coordination event, it lasts 400 time steps. So, $\mathcal{P}_\text{HM}=\{P_1,\dots,P_{20}\}$ s.t.  $P_i=(P^0_i,\dots,P^{400}_i)$. We vary $\rho \in \{0.25,0.50,0.75,1.00\}$. For each $\rho$, we generated 100 coordination events. In total, we have 400 datasets.
 
\subsubsection{Local Reversible agreement system}
 
 We created 100 other datasets for the LRA system. We used a set of strategy function $\mathcal{H}_\text{LRA}=\{f_i\}$ to generate $\mathcal{P}_\text{LRA}$ where all $f_i$ is $f_\text{LRA}$ (Eq.~\ref{eq:LRAStrFunc}).  For each dataset, it contains a set of time series of positions from 20 agents,  $\mathcal{P}_\text{LRA}=\{P_1,\dots,P_{20}\}$, where $P_i=(P^0_i,\dots,P^{400}_i)$. All agents updates their state $S^t_i$ corresponding to their local neighbors' states using a strategy function $f_{\text{LRA}}$.
 
\subsubsection{Hierarchical and Local Reversible agreement system}
 
  We created  100 other datasets of HM \& LRA coordination events by $\mathcal{H}_\text{HM\&LRA}$. We use this simulation to represent the group that has a coordination interval even if some agents use the HM strategy function but others use the LRA strategy function. For each dataset, it contains a set of position time series from 20 agents, $\mathcal{P}_\text{HM \& LRA}=\{P_1,\dots,P_{20}\}$, where $P_i=(P^0_i,\dots,P^{400}_i)$. The ID(1) is the informed agent. Agents who possess ID(2-10) use $f_{\text{HM}}$ with  $\rho = 1.00$. The rest of ID(11-20) agents use $f_{\text{LRA}}$.
  
\subsubsection{Mixed strategy system}
Lastly, we created 100 other datasets of mixed strategy of coordination events.  For each dataset, it contains a set of 20-agent position time series $\mathcal{P}_\text{MIX}=\{P_1,\dots,P_{20}\}$ where $P_i=(P^0_i,\dots,P^{400}_i)$ is time series of positions of agent $i$. The ID(1) is the informed agent. Other agents updates their state $S^t_i$ corresponding to both $f_{\text{HM}}$ with probability $0.5$ and $f_{\text{LRA}}$ with probability $0.5$.

 \subsubsection{Evaluation}
 \label{sec:eval}
 In this section, we evaluate the task of inference of the latent strategies given that we know the set of possible strategies. 
For each model, we performed 10-fold cross validation to evaluate the performance. For each round of cross validation, we have 100 datasets that can be separated into  45 training datasets, 45 validation datasets, and 10 testing datasets. We concatenated all time series in $\mathbb{P}_\text{test}$ to be a single time series $\mathcal{P}_\text{test}$ and also concatenate $\mathbb{S}_\text{test}$ to be $\mathcal{S}_\text{test}$. Then we use $\mathcal{P}_\text{test},\mathcal{S}_\text{test}$ to evaluate the direction prediction performance. We compare four strategy functions: $f_{\text{HM}}$, $f_{\text{LRA}}$, $f_{\text{AR}}$, and $f_{\text{OPT}}$, which is our framework optimal strategy function derived from Eq.~\ref{eq:Mix3Str} and \ref{eq:OPTsup}. We use the risk function that has Eq.~\ref{eq:DegreeDifFunc} as a loss function to evaluate the model performance. 
 
 \begin{equation}
 risk(\mathcal{P},\mathcal{S},f,i) = \frac{1}{T}\sum_{t=1}^T \text{DIST}_\text{dir} (S^t_i , f(S^{t-1},P^t,i))
 \label{eq:RiskDir}
 \end{equation}
 
 For each agent $i$, the best fitting model is the model that minimizes the risk function $risk(\mathcal{P}_\text{test},\mathcal{S}_\text{test},f,i)$ in Eq.~\ref{eq:RiskDir}.
\begin{equation}
   f^*_i= \argmin_{ f \in \{f_\text{HM},f_\text{LRA},f_\text{AR},f_\text{OPT}\}} risk(\mathcal{P}_\text{test},\mathcal{S}_\text{test},f,i) 
\end{equation}
 
 For each strategy function $f$, we report the distribution of loss values of direction prediction from all agents in each time step as well as the group's average optimal weight $\vec{w}^*_i$ from Eq.~\ref{eq:OPTsup}. If the framework performs well, then it should give the highest weight for the model that generated the dataset. 
 
 \subsection{Baboon behavioral experiment}
 
The dataset is the recording of GPS collars of an olive baboon (\emph{Papio anubis}) troop in the wild in Mpala Research Centre, Kenya~\cite{strandburg2015shared}. The GPS was recorded at 1 Hz from 7am until 7pm. 
The dataset consists of 16 individuals whose GPS trackers remained functional for 10 days.  A 2-dimensional trajectory of latitude and longitude for each individual has a length of 419,095 time steps. We extracted coordination events by FLICA varying the network density threshold at 25th, 50th, 75th, and 99th percentile and the time window at 240 time steps to infer coordination events and $60$ time steps to infer a dynamic following network. We used the 10-fold cross validation to report the results. For each round of cross validation, it has 45\% of training, 45\% of validation, and 10\% of testing coordination events. The remainder of the evaluation follows the description in the \textbf{Evaluation Section}. 
We use this experiment to demonstrate the ability of our framework to predict the next movement direction of agents even when the optimal strategy is unknown. The result can be used to generate (and test) hypotheses about the latent coordination strategies in collective movement data. 
 
 \subsection{Fish behavioral experiment}
We used the time series of  golden shiners (\emph{Notemigonus crysoleucas})  fish positions from \cite{strandburg2013visual}. The dataset was initially created to study information propagation via the fish visual fields~\cite{strandburg2013visual}. In total, there were 24 trails of fish position time series $\mathbb{P}=\{\mathcal{P}_1,\dots,\mathcal{P}_{24}\}$ in 2-dimensional space. For each $\mathcal{P}_k$, it consists of 70 fish, with 10 trained fish who are considered to be informed agents in our setting. On average, the time series in  $\mathcal{P}_k$ has its length around 600 time steps. The trained fish moved toward the feeding site (the target path) and the group follows them.  Due to the lack of information of identity for each individual in the different trails, we cannot train our framework in this dataset. Hence, we use fish data to demonstrate how to apply our framework to compare performance of each candidate strategy on direction prediction.

We compared the Informed strategy function $f_{\text{TF}}$ against $f_{\text{LRA}}$ in Eq.~\ref{eq:LRAStrFunc}. For each time step, $f_{\text{TF}}$ updates $S^t_i$ for any agent $i$ from the average of $S^t_j$ where $j$ is a trained fish. We use the risk function in Eq.~\ref{eq:RiskDir} to compare the performance among these strategy functions. For each strategy function $f$, we report the distribution of all agents' direction prediction error in each time step from $\text{DIST}_\text{dir} (S^t_i , f(S^{t-1},P^t,i)$.

\subsection{Comparison with the state of the art method}
Our method is the first approach to infer individual-level strategies that lead to group-level coordination. Thus, we compare our framework with the-state-of-the-art method, FLICA~\cite{FLICAtkdd}, for the task of leadership model classification. Since FLICA cannot infer the individual-level strategy, we evaluate both frameworks at the group-level classification task. We use simulated datasets from Section~\ref{sec:simdata}. 
Each set of time series has its label from one of the four models: HM, LRA, HM \& LRA, and Mix strategy model. FLICA maps each set of time series to the leadership ranking and convex hull features. In our framework, we use the median of $ \vec{w}^*_i$ (Eq.~\ref{eq:OPTsup}) to represent the feature vector of each dataset. We use 10-fold cross validation on Random Forests \cite{ho1998random} to report the evaluation results for both frameworks. To evaluate results, we define true positive (TP), false positive (FP), and false negative (FN) cases as follows. TP is the case when the predicted and ground-truth models of the dataset are the same. FP of model X is the case that a dataset that is not generated by model X is predicted as a model X's dataset. 
FN of model X is the case when a dataset of model X is predicted to be a model that is not X. We use TP, FP, and FN to calculate precision, recall, and F1 score to report results.
\section{Results}
\subsection{Simulations}

\begin{table}[th!]
\centering
\caption{The result of predicting the direction of movement via 10-fold cross validation. We compared the result of our framework (OPT) against the base-line pure strategies: HM, LRA, and AR (auto regressive strategy).
(*indicates the STD $\geq 20^\circ$)}
\label{tb:SimBaboonRes}
\begin{footnotesize}
\begin{tabular}{c|c|c|c|c|}
\cline{2-5}
\multicolumn{1}{l|}{}                                   & \multicolumn{4}{c|}{\begin{tabular}[c]{@{}c@{}}Average degree prediction error \\ {[}$0^\circ,180^\circ${]}\end{tabular}} \\ \hline
\multicolumn{1}{|c|}{Datasets\textbackslash Strategies} & OPT                             & HM                             & LRA                            & AR                    \\ \hline
\multicolumn{1}{|c|}{HM }                         & \textbf{12.40}                  & \textbf{12.98}                 & 20.49                          & 30.21*                \\ \hline
\multicolumn{1}{|c|}{LRA }                        & \textbf{7.77*}                  & 16.93*                         & \textbf{7.76*}                 & 13.78*                \\ \hline
\multicolumn{1}{|c|}{HM \& LRA }                  & \textbf{4.42}                   & 13.39*                         & 13.59                          & 23.87*                \\ \hline
\multicolumn{1}{|c|}{Mixed Str.}             & \textbf{29.33*}                 & 30.53*                         & 31.69*                         & 46.28*                \\ \hline
\multicolumn{1}{|c|}{Random}             & \textbf{89.74*}                 & 90.11*                         & \textbf{89.70*}                         & 90.21*                \\ \hline
\multicolumn{1}{|c|}{Baboon}             & \textbf{53.16*}                & \textbf{53.16*}                         & 72.36*                         & 85.84*                \\ \hline
\end{tabular}
\end{footnotesize}
\end{table}

The results of inferring the coordination strategy in simulated datasets are shown in Table~\ref{tb:SimBaboonRes}. A row represents the results from datasets generated by a specific model. A column represents a strategy prediction error measured in degree units $[0^\circ,180^\circ]$. OPT is the optimal strategy function trained by our framework. HM is Eq.~\ref{eq:HMstrFunc2}. LRA is Eq.~\ref{eq:LRAStrFunc}. AR is the auto regressive strategy function that chooses the current direction $t$ based on the previous five time steps from the same agent. We use AR as the baseline. In all datasets, our framework (OPT column) has roughly smallest error among all other strategies. For the first two rows of HM and LRA datasets, OPT has almost the same performance as the strategies used to generate the data  (HM row/column and LRA row/column).  For HM \& LRA datasets in the third row, each individual might use either HM or LRA strategy. Hence, using the homogeneous strategy to predict directions for all agents results in larger error (HM and LRA column). On the contrary, our framework can detect which individual uses which strategy. Hence, OPT performed better than all pure strategies. Similarly, for the mixed strategy datasets (Mixed Str. row), each individual might use either HM or LRA as its strategy with the probability 0.5. Since our framework can infer mixed strategies,  it performed better than using any pure strategy.   Lastly, we reported the results of the direction prediction from the 100 datasets of time series generated from $n$ agents moving uniformly and randomly in any direction (Random row). The result shows that all strategies included in our framework produced the same bad result with the loss value at $90^\circ$ degree. This shows that our framework does not find an artifact model where none exists.

\begin{table}[th]
\caption{The average optimal support vector $\vec{w}$ of all agents from 10-fold cross validation, inferred by our framework from simulated and the Baboon datasets. }
{%
\begin{center}
\begin{footnotesize}
\begin{tabular}{c|c|c|c|}
\cline{2-4}
\multicolumn{1}{l|}{}                                                                        & \multicolumn{3}{c|}{Average Support $\vec{w}$ (predict/actual)} \\ \hline
\multicolumn{1}{|c|}{Datasets}                                                               & $w_1$:HM             & $w_2$:LRA            & $w_3$:AR  \\ \hline
\multicolumn{1}{|c|}{HM}                                                              & \textbf{0.85/1.00}   & 0.12/0.00            & 0.03/0.00         \\ \hline
\multicolumn{1}{|c|}{LRA}                                                             & 0.02/0.00            & \textbf{0.98/1.00}   & 0.00/0.00         \\ \hline
\multicolumn{1}{|c|}{\begin{tabular}[c]{@{}c@{}}HM \& LRA\\ (HM part)\end{tabular}}   & \textbf{1.00/1.00}   & 0.00/0.00            & 0.00/0.00         \\ \hline
\multicolumn{1}{|c|}{\begin{tabular}[c]{@{}c@{}}HM \& LRA \\ (LRA part)\end{tabular}} & 0.00/0.00            & \textbf{1.00/1.00}   & 0.00/0.00         \\ \hline
\multicolumn{1}{|c|}{Mixed Strategy}                                                  & \textbf{0.48/0.50}   & \textbf{0.48/0.50}   & 0.04/0.00         \\ \hline
\multicolumn{1}{|c|}{Random}                                                          & {0.09/0.00}   & 0.86/0.00            & 0.05/0.00         \\ \hline
\multicolumn{1}{|c|}{Baboon}                                                         & 1.00/NA              & 0.00/NA              & 0.00/NA            \\ \hline
\end{tabular}
\end{footnotesize}
\end{center}
}
\label{tb:SimeWopt}
\end{table}

Table~\ref{tb:SimeWopt} shows the support vectors for each strategy corresponding to the  datasets in Table~\ref{tb:SimBaboonRes} in the OPT column. For each element in the table, the first number is the predicted support from our framework and the second is the actual support that we used to create the datasets. For example, in the first element of HM row, 0.85/1.00 means we used HM strategy to create HM datasets and the framework inferred the HM support  in these datasets as 0.85. Overall, our framework correctly inferred the support vectors of all non-random datasets, while avoiding overfitting.

\subsection{Baboon behavioral experiment}

We varied the threshold of the following network density to infer coordination events in the baboon dataset. We report the average result from all the thresholds. The last row of Table~\ref{tb:SimBaboonRes} shows the result of the direction prediction of baboons, using different coordination strategies. The OPT coordination strategy, as derived by our framework, is in the last row in Table~\ref{tb:SimeWopt}. According to the result, OPT used HM as the pure strategy. The errors of HM and OPT strategies suggest that baboons may have a slight preference to follow a pre-determined individual or a set of individuals, rather than their neighbors in the position space. This is consistent with the biological understanding of the baboon social behavior~\cite{farine2016both}. However, the more accurate strategy should be investigated and biologically verified.

\subsection{Fish behavioral experiment}

\begin{table}[th!]
\centering
\caption{Comparison between LRA and Informed strategies to predict directions of 24 trails of fish\protect\footnotemark.}  
\label{tb:FishRes}
\begin{footnotesize}
\begin{tabular}{c|c|c|}
\cline{2-3}
\multicolumn{1}{l|}{}                   & \multicolumn{2}{l|}{Error of degree prediction {[}$0^\circ,180^\circ${]}} \\ \hline
\multicolumn{1}{|c|}{Strategies}        & Mean                                & STD                                 \\ \hline
\multicolumn{1}{|c|}{LRA}               & 41.51                      & 45.11                               \\ \hline
\multicolumn{1}{|c|}{Informed Strategy} & 54.46                               & 47.68                      \\ \hline
\end{tabular}
\end{footnotesize}
\end{table}

The results of the direction prediction in fish datasets, for LRA and Informed strategies, are in Table~\ref{tb:FishRes}. The LRA performed better than the Informed strategy, indicating that fish follow their immediate neighbors in space. This result is supported by the work in \cite{strandburg2013visual,Katz18720} and many others,  showing that fish do not directly know who leads the group but follow their neighbors.\footnotetext{The reason that fish datasets have their own table while other datasets are in another table is because of the following reason. To use 10-fold cross validation, we have to be able to learn each individual strategy from one set of coordination events (training datasets) to predict the strategy of the same individual in another set of coordination events (validation datasets). In fish datasets, there are 24 fish coordination events. However, fish datasets lack of individual identities. Precisely, two individuals with the same ID from two different fish-coordination events might not be the same individual. In contrast, two individuals with the same ID from two different coordination events are always the same individual in both simulation and baboon datasets.  Hence, we cannot use 10-fold cross validation procedure on fish datasets the same way as we did on baboon and simulation datasets.}

\subsection{Comparison with the state of the art method}

\begin{table}[]
\centering
\caption{The results of model classification of FLICA and the proposed framework via 10-fold cross validation. We use Random Forest for classification.}
\label{tb:modelClsRes}
\begin{footnotesize}
\begin{tabular}{c|c|c|c|c|c|c|}
\cline{2-7}
\textbf{}                              & \multicolumn{3}{c|}{\textbf{FLICA}}          & \multicolumn{3}{c|}{\textbf{Proposed Method}} \\ \hline
\multicolumn{1}{|c|}{\textbf{Classes}} & \textbf{Precision} & \textbf{Recall} & \textbf{F1 score} & \textbf{Precision}   & \textbf{Recall}  & \textbf{F1 score}  \\ \hline
\multicolumn{1}{|c|}{\textbf{HM}}      & 1              & 0.75          & 0.86        & 1                & 1              & 1            \\ \hline
\multicolumn{1}{|c|}{\textbf{LRA}}     & 0.8            & 1             & 0.89        & 1                & 1              & 1            \\ \hline
\multicolumn{1}{|c|}{\textbf{HM \& LRA}}  & 0.94           & 1             & 0.97        & 0.98             & 1              & 0.99         \\ \hline
\multicolumn{1}{|c|}{\textbf{Mixed Str.}}     & 0.90           & 0.94          & 0.92        & 1                & 0.98           & 0.99         \\ \hline
\multicolumn{1}{|c|}{\textbf{Random}}    & 1              & 0.9           & 0.95        & 1                & 1              & 1            \\ \hline
\end{tabular}
\end{footnotesize}
\end{table}

The result of model classification using FLICA as well as the proposed framework is in Table~\ref{tb:modelClsRes}. In all datasets, the proposed framework performed better than FLICA. This indicates that the group-level features that FLICA provides for classification are not sufficiently informative to be used to categorize complicated datasets where individuals may use a heterogeneous set of strategies (e.g. HM \& LRA). 

For the baboon and fish datasets, since there is no ground truth available regarding classes of strategies, we can only discuss the results of both datasets from FLICA and the new insight from our proposed framework here. The FLICA result of classification in~\cite{FLICAtkdd} stated that baboons used a linear threshold model to form coordination; there is no association of orders of movement velocity and position of individuals vs. ranking of movement initiation. In other words, initiators do not necessary move first or in a front of a group. In this work, Table~\ref{tb:SimBaboonRes} suggests that there is a hierarchy among baboons; baboons trend to follow the directions of specific individuals. This result is consistent with the result in~\cite{Amornbunchornvej2019} that performed analysis on the same baboon dataset, which showed that there are several pairs of baboons that follow each other with high supports in various situations. For the fish datasets, the result of FLICA framework~\cite{FLICAtkdd} suggests that trained fish truly initiated coordination movement. In this work, Table~\ref{tb:FishRes} suggests that schools of fish used LRA strategy; individuals in school of fish do not follow trained fish directly, but they follow their neighbors. 
\section{Limitations and future work} 
Even though our framework can address \cmip, there are several limitations in our work. First, our framework considers only three types of strategies: HM, LRA, and AR. Hence, our framework can distinguish only whether each agent follows its neighbors (LRA), specific individuals (HM), or itself (AR). Second, our framework assumes that there is only one target path that a group tries to form coordination with. In the case of multiple target paths, we need another framework (e.g. mFLICA~\cite{amornbunchornvej2018framework}) to segment each faction of coordination events that has a different target path, then applying our proposed framework to infer individual strategies. Third, our framework considers each agent as a point without considering environmental factors (e.g. obstacles, gap between agents before collision, constrains of movement). Lastly, our framework assumes that a state of an agent at time $t$ is affected by previous states of other agents and/or itself at time $t-1$.  These limitations enable opportunities for future research.
\section{Conclusions}
In this paper, we formalized a new computational problem, \cmip. Given a set of candidate strategies and a set of time series of coordinated movement as inputs, our goal is to infer the original strategy that each individual used to achieve the group coordination. We showed that a strategy that has the convergence property can guarantee that the group reaches coordination. We provide the first approach to infer the set of strategies that each individual uses to achieve movement coordination at the group level. We evaluated and demonstrated our framework performance in simulated datasets as well as two biological datasets: baboon and fish. Our framework was able to infer the original set of strategy functions that generated each simulated dataset. The results show that our approach is highly accurate in inferring the correct strategy in simulated datasets even in complicated mixed strategy settings.  Moreover, our framework performed classification of group-level coordination models from time series better than FLICA framework, which is the-state-of-the-art approach for the task. Animal data experiments show that fishes, unsurprisingly, follow their neighbors, while baboons have a preference to follow specific individuals.  
Although we used the specific setting of focusing on the direction of movement as the definition of an agent's state and used three exemplar candidate strategy, our methodology easily generalizes to arbitrary time series data on Euclidean space, beyond movement data, and other candidate strategies. While for the fairness of comparison with the biological datasets we used simulated data of 20 individuals, it is clear that there are no inherent limitations in the approach to scale to much larger datasets. The only barrier is the availability of data. The code and datasets that we used in this paper can be found at ~\cite{ShareSourcecode}.

\balance
\bibliographystyle{ACM-Reference-Format}

\balance

\end{document}